\documentclass[runningheads]{llncs}

\usepackage{amsmath,amssymb,amsfonts}
\usepackage{graphicx}

\usepackage[T1]{fontenc}
\usepackage{microtype}
\usepackage{xcolor}

\usepackage{hyperref}
\usepackage{url}

\newcommand{\kB}{k_{\mathrm{B}}}

\begin{document}

\title{Thermodynamic Limits of Physical Intelligence}
\titlerunning{Thermodynamic Limits of Physical Intelligence}
\author{Koichi Takahashi\inst{1,2} \and
Yusuke Hayashi\inst{1}}
\authorrunning{K. Takahashi and Y. Hayashi}
\institute{AI Alignment Network (ALIGN) \and
RIKEN Advanced General Intelligence for Science (AGIS) Program\\
\email{ktakahashi@riken.jp}}

\maketitle

\begin{abstract}
Modern AI systems achieve remarkable capabilities at the cost of substantial energy consumption. To connect intelligence to physical efficiency, we propose two complementary bits-per-joule metrics under explicit accounting conventions: (1) Thermodynamic Epiplexity per Joule, new bits of structure about a specified environment-instance variable encoded in an agent's state per unit energy, and (2) Empowerment per Joule, sensorimotor channel capacity per expected energetic cost over a fixed horizon. These give two axes of physical intelligence, recognition vs.\ control, but the resulting numbers are benchmark-relative rather than universal. Drawing on stochastic thermodynamics, we formulate a Landauer-scale closed-cycle benchmark for epiplexity acquisition by combining a thermodynamic-learning inequality with data processing, and clarify why boundary closure is required; conversely, a decoupling construction shows that without such assumptions information gain and in-boundary dissipation need not be tightly linked. For empirical settings where the latent structure variable is unavailable, we recommend compute-bounded MDL epiplexity / compression-gain surrogates. Finally, we propose a unified efficiency framework with a minimal checklist of conventions for relative bits-per-joule comparisons, and give a compact language-model reporting example.
\end{abstract}

\section{Introduction}

AI capabilities have grown dramatically in recent years, but at the expense of equally dramatic increases in energy consumption. Large-scale training runs for foundation models consume megawatt-hours of electricity, prompting concerns about the environmental footprint and physical sustainability of advanced AI. Biological intelligence, by contrast, achieves high-level cognition with only $\sim 20\,\mathrm{W}$ (the human brain)---orders of magnitude more efficient. This gap motivates a fundamental question: How can we formally define and quantify the energy efficiency of an intelligent system, and what thermodynamic limits constrain it? Answering this may also help forecast what lies beyond current AI scaling trends by clarifying the role of thermodynamic constraints \cite{takahashi2023scenarios}.

This paper targets a reproducible efficiency report rather than a universal intelligence score: comparisons are only meaningful under explicit boundary, coarse-graining/noise, horizon/reset, and cost conventions. We therefore distinguish two uses of the metrics. As absolute physical-efficiency targets, the Landauer-scale limits are conceptual yardsticks far beyond current hardware. As relative comparison tools across architectures, embodiments, or training procedures at fixed benchmarks and operating points, the same bits/J ratios are immediately actionable.

\subsection{Accounting and measurement conventions}\label{sec:accounting}
Bits-per-joule metrics depend on what energy flows are counted. Throughout, we distinguish (i) the measured energy consumption inside an explicit accounting boundary, $E_{\mathrm{cons}}$ (J), from (ii) the thermodynamic dissipation $Q_{\mathrm{diss}}$ (heat released to an isothermal bath) that appears in stochastic-thermodynamic inequalities. For a generic episode, an energy balance can be written as
\begin{align}
  E_{\mathrm{cons}}
  \;=\;
  Q_{\mathrm{diss}}
  \;+
  \Delta U_{\mathrm{sys}}
  \;+
  W_{\mathrm{out}}
  \;+
  \Delta E_{\mathrm{store}},
  \label{eq:energy_balance}
\end{align}
where $\Delta U_{\mathrm{sys}}$ is the change in internal energy of degrees of freedom inside the boundary, $W_{\mathrm{out}}$ is exported work (useful work delivered outside the boundary), and $\Delta E_{\mathrm{store}}$ captures energy stored and later recovered (e.g.\ batteries, springs, capacitors, or mechanical potential energy). When $\Delta U_{\mathrm{sys}}$, $W_{\mathrm{out}}$, and $\Delta E_{\mathrm{store}}$ are negligible over the evaluation interval, one can justify the approximation $E_{\mathrm{cons}}\approx Q_{\mathrm{diss}}$; otherwise, substituting $E_{\mathrm{cons}}$ for $Q_{\mathrm{diss}}$ is a reporting convention and the additional terms should be reported (or bounded) explicitly.
Finally, the accounting boundary must include externally prepared low-entropy resources (e.g.\ freshly initialized memory and its maintenance) if one wishes to compare against Landauer-scaled benchmarks in repeated operation.
Throughout, when we invoke Landauer-scaled benchmarks we mean closed-cycle repeated-operation regimes with explicit boundary closure; the open-boundary decoupling construction (Proposition~\ref{prop:open_boundary_decoupling}) is included only as an accounting caution motivating this requirement.

\subsection{Related work}
Several recent works link intelligence and thermodynamics \cite{arxiv2502_15820}. WPI \cite{perrier2025wpi} connects task performance and energy via Landauer's principle \cite{landauer1961irreversibility}, whereas classical measures such as Legg--Hutter universal intelligence \cite{legg2007universal} and broader intelligence evaluation frameworks \cite{hernandez2017measure} abstract away physical resources. These approaches are informative but do not directly yield a standardized bits-per-joule comparison that separates learning (recognition/model building) from control (action influence) under explicit accounting conventions.
Thermodynamic learning inequalities relate information acquisition in stochastic learning dynamics to entropy production under explicit subsystem assumptions \cite{goldt2017thermodynamic}, while related work on the thermodynamics of prediction relates dissipation to storing non-predictive information \cite{still2012thermodynamics}. Empowerment formalizes task-agnostic control capacity as a sensorimotor channel capacity \cite{klyubin2005empowerment}. Our contribution is to unify these strands into a two-axis reporting protocol with the stated conventions.
Epiplexity as compute-bounded structural information was recently formalized via time-bounded MDL \cite{epiplexity2026}. That work defines epiplexity as the model-description component of a resource-bounded two-part code (with a companion time-bounded entropy term) and provides practical estimators (e.g.\ prequential coding) for modern ML systems. Because this operational companion is recent, we use it as empirical scaffolding; the thermodynamic bounds themselves rest on established stochastic-thermodynamic and information-theoretic results. In this paper, we use $I(W;Z)$ only as a normative target for controlled environment-instance benchmarks, and we adopt compute-bounded MDL epiplexity as the default operational companion when $Z$ is unavailable.

\paragraph{Contributions}
\begin{itemize}
  \item \textbf{Thermodynamic epiplexity per joule.} We define $\eta_{\mathcal{E}}\triangleq\Delta\mathcal{I}/E_{\mathrm{cons}}$ as a learning-efficiency metric (bits/J) for how much new environment structure is retained in an agent's internal state within an explicit accounting boundary. The conditioning on $W^{\mathrm{pre}}$ makes $\Delta\mathcal{I}$ an acquired, episode-level quantity rather than cumulative stored information. The closed-cycle Landauer-scale statement is a synthesis of known ingredients---thermodynamic learning plus conditional data processing---applied to AI-efficiency evaluation; when the latent structure variable is unavailable, we recommend compute-bounded MDL/compression surrogates \cite{epiplexity2026}.
  \item \textbf{Empowerment per joule.} We define empowerment as an embodied sensorimotor channel capacity over horizon $\tau$ and define $\eta_{\mathcal{C}}$ via a cost-constrained empowerment curve (and derived bits/J summaries), with reporting conventions (total vs.\ incremental energy) that avoid free-control artifacts; we relate this to classical capacity-per-unit-cost results \cite{verdu1990capacity}.
  \item \textbf{Unified efficiency framework.} We propose a unified efficiency framework for physical AI agents (Section~\ref{sec:unified_framework}) that jointly reports $\eta_{\mathcal{E}}$ and $\eta_{\mathcal{C}}$, together with a minimal reporting checklist (boundary and energy accounting, coarse-graining/noise, horizon/reset, cost baseline, and estimator details) to support consistent bits-per-joule comparisons.
\end{itemize}

\section{Thermodynamic Epiplexity: Learning Bits per Joule}

\subsection{Defining Epiplexity: Normative vs.\ Operational}
Intelligent agents improve their performance by learning regularities---statistical or causal structure---from interaction data. We use \emph{epiplexity}\footnote{The term is intended to evoke structure ``around'' an observed stream, complementing entropy: epiplexity tracks model structure rather than unexplained randomness.} to denote the amount of \emph{environmental structure encoded in an agent's internal state}, but we distinguish two layers: (i) a \emph{normative} mutual-information target relative to a benchmark-provided environment-instance variable $Z$, and (ii) an \emph{operational} compute-bounded MDL notion of structural information (also called epiplexity) designed for computationally bounded observers \cite{epiplexity2026}. This two-layer view lets us state thermodynamic benchmarks for a clean theoretical quantity while retaining a reproducible empirical companion when $Z$ is not available.

\paragraph{Generative-environment setup.}
To state a normative theoretical target for ``structure,'' we adopt a standard \emph{generative model} viewpoint. An ``environment instance'' is parameterized by a latent random variable $Z\sim p(z)$, and (conditional on $Z$) the agent's experience $X$ is generated according to an environment model $p_{\mathrm{env}}(x\mid Z)$. In interactive settings, $X$ can represent a trajectory and we can write $p_{\mathrm{env}}(o_{0:\tau-1}\mid a_{0:\tau-1}, Z)$; in either case, $Z$ collects the latent degrees of freedom (e.g.\ transition parameters, latent causes, or a causal graph) that govern the data-generating process. We emphasize that $Z$ is a theoretical variable for analysis and need not be observable in typical empirical benchmarks.

\paragraph{Choice of $Z$ and comparability.}
The numerical value of $I(W;Z)$ is defined relative to a specified notion of ``environment instance,'' i.e., a chosen generative family $\{p_{\mathrm{env}}(\cdot\mid Z)\}$. In complex real-world domains there may be no unique, objective choice of such a $Z$; accordingly, we treat the choice of $Z$ (and any quotienting of redundant reparameterizations) as part of the benchmark specification rather than as something inferred from data. When comparing systems, one should hold this specification fixed.
Where feasible, we recommend defining $Z$ to be \emph{minimal} with respect to the induced distribution of agent-accessible trajectories $X$ under the benchmark protocol (e.g.\ collapsing observationally equivalent instances), to reduce redundancy and discourage gaming by adding extraneous degrees of freedom.
In empirical benchmarks where $Z$ is not accessible, we do not attempt to estimate this normative mutual information and instead report compute-bounded MDL epiplexity / compression-gain companions under stated modeling and resource conventions (below) \cite{epiplexity2026}.

\paragraph{Definition (mutual-information epiplexity).}
Let $W$ denote the agent's internal state (weights, memory, beliefs). After processing experience under a learning resource budget $B$, the internal state changes from $W^{\mathrm{pre}}$ to $W^{\mathrm{post}}$ (where $B$ may include compute, time, or data). We define the agent's epiplexity after learning as the mutual information between its state and the environment's latent structure:
\begin{align}
  \mathcal{I} \;\triangleq\; I(W; Z)\qquad [\text{bits}].
\end{align}

\paragraph{Episode-level acquired epiplexity.}
The difference $I(W^{\mathrm{post}};Z)-I(W^{\mathrm{pre}};Z)$ can be negative if the update forgets or overwrites previously learned structure. To obtain a nonnegative ``how many new bits were acquired in this episode'' quantity, we define \emph{acquired epiplexity} as a conditional mutual information:
\begin{align}
  \Delta \mathcal{I} \;\triangleq\; I\!\left(W^{\mathrm{post}}; Z \,\middle|\, W^{\mathrm{pre}}\right),
\end{align}
which measures how many bits about $Z$ are newly encoded in $W^{\mathrm{post}}$ beyond what was already present in $W^{\mathrm{pre}}$. This quantity is inherently observer-dependent through $(W,Z)$ and the budget $B$: a more capable learner (larger $B$, richer inductive bias) can encode more information about $Z$.

\paragraph{$\varepsilon$-coarse-grained epiplexity.}
If $W$ or $Z$ are continuous, we define a coarse-grained version by applying a fixed quantizer $Q_\varepsilon(\cdot)$ (or an observation-noise model of scale $\varepsilon$) and computing mutual information on the resulting discrete variables:
\begin{align}
  \mathcal{I}_\varepsilon \triangleq I\!\left(Q_\varepsilon(W); Q_\varepsilon(Z)\right),
  \qquad
  \Delta \mathcal{I}_\varepsilon \triangleq I\!\left(Q_\varepsilon(W^{\mathrm{post}}); Q_\varepsilon(Z)\mid Q_\varepsilon(W^{\mathrm{pre}})\right).
\end{align}
Unless stated otherwise, we use $\Delta \mathcal{I}$ as shorthand for the appropriate (possibly coarse-grained) acquired-epiplexity quantity.

This definition targets information about the environment instance rather than raw dataset length or redundancy.

\paragraph{Operational epiplexity via compute-bounded MDL.}
Following Finzi et al.~\cite{epiplexity2026}, let $\mathcal{M}_B$ denote a model/program class constrained by a stated resource budget $B$ (e.g.\ runtime, memory, or training budget). Define the (two-part) resource-bounded MDL objective
\begin{align}
  \mathrm{MDL}_B(X)
  &\;\triangleq\;
  \min_{M\in\mathcal{M}_B}\bigl[L(M)+L(X\mid M)\bigr],\\
  M_B^\star &\in \arg\min_{M\in\mathcal{M}_B}\bigl[L(M)+L(X\mid M)\bigr],
\end{align}
where $L(M)$ is the model-description length (structural bits) and $L(X\mid M)$ is the data-given-model code length (random/unexplained bits) in bits under an explicit coding convention.
The corresponding MDL epiplexity (structural component) and time-bounded entropy are
\begin{align}
  \mathcal{I}_B^{\mathrm{MDL}}(X)&\triangleq L(M_B^\star),
  \qquad
  H_B^{\mathrm{MDL}}(X)\triangleq L(X\mid M_B^\star),\\
  \mathrm{MDL}_B(X)&=\mathcal{I}_B^{\mathrm{MDL}}(X)+H_B^{\mathrm{MDL}}(X).
\end{align}
In empirical settings where $Z$ is unavailable, we report $\mathcal{I}_B^{\mathrm{MDL}}$ and/or the induced compression gains under stated coding, budget, and training conventions as operational companions to the normative $\mathcal{I}=I(W;Z)$.

\paragraph{Relation via data processing.}
To situate the normative target $\mathcal{I}=I(W;Z)$ relative to episode data, observe that the learning update uses the episode data $X$: conditioned on $W^{\mathrm{pre}}$, we have the Markov relation $Z \to X \to W^{\mathrm{post}}$. By the (conditional) data processing inequality,
\begin{align}
  \Delta \mathcal{I}
  = I(W^{\mathrm{post}};Z\mid W^{\mathrm{pre}})
  \;\le\;
  I(X;Z\mid W^{\mathrm{pre}}).
\end{align}
In passive/batch settings where $W^{\mathrm{pre}}$ does not influence the data-generating process for $X$ (e.g.\ a fixed data collection policy independent of $W^{\mathrm{pre}}$), this reduces to the familiar upper bound $\Delta \mathcal{I} \le I(X;Z)$. In interactive settings, $W^{\mathrm{pre}}$ can affect $X$ through actions (active learning), so reporting conventions must specify the policy class and horizon over which $X$ is generated.
MDL epiplexity \cite{epiplexity2026} can serve as a reproducible operational companion under stated coding and resource budgets, but we assume no equality with $\mathcal{I}=I(W;Z)$ without additional modeling assumptions.

Unlike reward-based ``intelligence'' measures (e.g.\ Legg--Hutter), epiplexity quantifies structural information about $Z$ without assuming an external reward signal; high epiplexity correlates with generalization but does not by itself guarantee goal-directed success.

\subsection{Epiplexity per Joule: Efficiency of Learning}

We now link epiplexity to physical energy costs. Thermodynamic epiplexity per joule $\eta_{\mathcal{E}}$ is:
\begin{align}
  \eta_{\mathcal{E}}
  \;\triangleq\;
  \frac{\Delta \mathcal{I}}{E_{\mathrm{cons}}}
  \;=\;
  \frac{I\!\left(W^{\mathrm{post}};Z\mid W^{\mathrm{pre}}\right)}{E_{\mathrm{cons}}}
  \quad [\text{bits per joule}].
\end{align}
Here $\Delta \mathcal{I}$ is the episode-level acquired epiplexity and $E_{\mathrm{cons}}$ is the measured energy inside the accounting boundary (Section~\ref{sec:accounting}). When $W$ or $Z$ are continuous, $\Delta \mathcal{I}$ should be read as $\Delta \mathcal{I}_\varepsilon$ for a stated coarse-graining.

In an ideal closed-cycle benchmarking regime, Landauer's principle \cite{landauer1961irreversibility} sets a natural scale of $\kB T\ln 2$ joules per reliably reusable bit, i.e.\ on the order of $1/(\kB T\ln 2)$ bits per joule. Approaching this benchmark requires that most expended energy contributes to retained model information with minimal overhead.
At room temperature $T\approx 300\,\mathrm{K}$, $1/(\kB T\ln 2)\approx 3.5\times 10^{20}$ bits/J (a far-above-hardware benchmark scale).
Under standard thermodynamic-learning assumptions, Corollary~\ref{cor:closed_cycle_epiplexity_limit} makes this Landauer-scale benchmark explicit for dissipation-normalized learning efficiency in steady state.

\paragraph{Accounting conventions and when Landauer-type bounds apply.}
The Landauer scale $\kB T\ln 2$ is best interpreted as a \emph{benchmark} for closed information-pro\-cessing cycles in which (i)~the relevant memory degrees of freedom are \emph{reused} (bounded effective memory), (ii)~the agent's internal state is returned to a standard distribution (full or partial resets), and (iii)~stored bits have a specified reliability against thermal noise. In such \emph{closed-cycle} settings, irreversibility cannot be avoided on average and bits-per-joule is bounded on the order of $1/(\kB T\ln 2)$.

By contrast, in \emph{open-ended} settings where fresh low-entropy memory can be allocated without charge, the ratio of stored correlations to in-boundary dissipation can diverge via reversible computation. Fresh memory is itself a thermodynamic resource (negentropy); if its preparation cost lies outside the boundary, efficiency metrics become misleading. Proposition~\ref{prop:open_boundary_decoupling} provides a decoupling example and shows how closing the accounting boundary restores Landauer-scaled benchmarks.

For thermodynamic analysis it is also useful to report the dissipation-nor\-mal\-ized ratio
$\tilde{\eta}_{\mathcal{E}}
\triangleq
\Delta \mathcal{I}/\allowbreak Q_{\mathrm{diss}}$
alongside~$\eta_{\mathcal{E}}$; see Section~\ref{sec:accounting}.

\paragraph{Thermodynamic learning inequalities.}
In general, information--thermodynamic bounds involve the \emph{total entropy production} $\Sigma$, which (for an isothermal process at temperature $T$) can be written as
\begin{align}
  \Sigma \;\triangleq\; \Delta S_{\mathrm{sys}} + \frac{Q_{\mathrm{diss}}}{T}
  \qquad [\text{J/K}],
\end{align}
where $\Delta S_{\mathrm{sys}}$ is the \emph{physical} entropy change (J/K) of the learning subsystem.\footnote{If $H(W)\triangleq -\sum_w p(w)\log_2 p(w)$ is the Shannon entropy in bits, then $S_{\mathrm{sys}}=\kB\ln 2\,H(W)$ and
$\Delta S_{\mathrm{sys}}=\kB\ln 2\,\Delta H(W)$.
Equivalently, $\Sigma/\kB$ is the dimensionless entropy production.}

\begin{lemma}[Thermodynamic learning inequality {\cite{goldt2017thermodynamic}}]\label{thm:thermo_learning}
Consider an isothermal stochastic learning dynamics at temperature~$T$ in which a learner state~$W$ is driven by an external data stream~$X$.
Assume (i)~$(W_t,X_t)$ forms a \emph{bipartite} Markov process (updates act on $W$ or $X$ but not simultaneously), (ii)~the $W$-subsystem obeys local detailed balance w.r.t.\ the bath at temperature~$T$ so that $Q_{\mathrm{diss}}$ over an episode is well-defined, and (iii)~$\Delta S_{\mathrm{sys}}$ denotes the physical entropy change of the $W$-subsystem. Then the information flow into the learner satisfies (in bits)
\begin{align}
  \Delta I_{W\leftarrow X}
  &\;\triangleq\;
  I\!\left(W^{\mathrm{post}};X \,\middle|\, W^{\mathrm{pre}}\right)
  \notag\\
  &\;\le\;
  \frac{\Delta S_{\mathrm{sys}} + Q_{\mathrm{diss}}/T}{\kB\ln 2}.
  \label{eq:gs_bound_bits}
\end{align}
This statement is a convenient episode-level restatement of thermodynamic learning inequalities derived by Goldt and Seifert~\cite{goldt2017thermodynamic}; the exact form depends on the subsystem choice and stochastic-thermodynamic modeling assumptions.
\end{lemma}

\paragraph{From data-information to structure-information.}
Under the generative setup $Z\to X\to W^{\mathrm{post}}$ (given $W^{\mathrm{pre}}$), the data processing inequality implies that acquired epiplexity is bounded by the information the learner acquires about the data:
\begin{align}
  \Delta \mathcal{I}
  = I(W^{\mathrm{post}};Z\mid W^{\mathrm{pre}})
  \;\le\;
  I(W^{\mathrm{post}};X\mid W^{\mathrm{pre}})
  = \Delta I_{W\leftarrow X}.
\end{align}
Therefore, whenever an inequality of the form in Lemma~\ref{thm:thermo_learning} holds for the learner's information gain about its driving signal, it also yields a valid (possibly loose) upper bound on $\Delta \mathcal{I}$.

\begin{corollary}[Closed-cycle epiplexity benchmark (Landauer scale)]\label{cor:closed_cycle_epiplexity_limit}
Assume the thermodynamic learning inequality in Lemma~\ref{thm:thermo_learning} holds for the information flow $\Delta I_{W\leftarrow X}$, and assume the generative Markov relation $Z\to X\to W^{\mathrm{post}}$ holds conditioned on $W^{\mathrm{pre}}$. Then the acquired epiplexity satisfies
\begin{align}
  \Delta \mathcal{I}
  = I(W^{\mathrm{post}};Z\mid W^{\mathrm{pre}})
  \;\le\;
  \frac{\Delta S_{\mathrm{sys}} + Q_{\mathrm{diss}}/T}{\kB\ln 2}.
  \label{eq:closed_cycle_epiplexity_bound}
\end{align}
Equivalently,
\begin{align}
  Q_{\mathrm{diss}}
  \;\ge\;
  \kB T\ln 2 \,\Delta \mathcal{I}
  \;-\;
  T\,\Delta S_{\mathrm{sys}}.
  \label{eq:q_lower_with_dS}
\end{align}
In particular, in a closed-cycle/steady-state regime with $\Delta S_{\mathrm{sys}}=0$ (per episode or on average), we obtain the Landauer-scale limit
\begin{align}
  \Delta \mathcal{I} \;\le\; \frac{Q_{\mathrm{diss}}}{\kB T\ln 2},
  \qquad\text{and hence (when $Q_{\mathrm{diss}}>0$)}\qquad
  \tilde{\eta}_{\mathcal{E}}
  \triangleq
  \frac{\Delta \mathcal{I}}{Q_{\mathrm{diss}}}
  \;\le\;
  \frac{1}{\kB T\ln 2},
\end{align}
under the stated coarse-graining and reliability conventions.
\end{corollary}

\begin{proof}
By the (conditional) data processing inequality under $Z\to X\to W^{\mathrm{post}}$ conditioned on $W^{\mathrm{pre}}$, we have
\[
\Delta \mathcal{I} = I(W^{\mathrm{post}};Z\mid W^{\mathrm{pre}})\le I(W^{\mathrm{post}};X\mid W^{\mathrm{pre}})=\Delta I_{W\leftarrow X}.
\]
Combining with Lemma~\ref{thm:thermo_learning} yields Eq.~\eqref{eq:closed_cycle_epiplexity_bound}, and rearranging yields Eq.~\eqref{eq:q_lower_with_dS}. Setting $\Delta S_{\mathrm{sys}}=0$ gives $\Delta \mathcal{I}\le Q_{\mathrm{diss}}/(\kB T\ln 2)$, which implies the stated bound on $\tilde{\eta}_{\mathcal{E}}$ whenever $Q_{\mathrm{diss}}>0$.
\end{proof}

\begin{remark}[Thermodynamic vs.\ measured-energy limits]
Corollary~\ref{cor:closed_cycle_epiplexity_limit} bounds epiplexity acquisition per \emph{dissipated heat} $Q_{\mathrm{diss}}$. Translating this into a bound on $\eta_{\mathcal{E}}=\Delta \mathcal{I}/E_{\mathrm{cons}}$ requires the energy accounting in Eq.~\eqref{eq:energy_balance}; in regimes where $E_{\mathrm{cons}}\approx Q_{\mathrm{diss}}$ (or where $E_{\mathrm{cons}}$ upper-bounds $Q_{\mathrm{diss}}$ under the stated boundary conventions), the same Landauer-scale limit applies to $\eta_{\mathcal{E}}$ up to the stated approximation.
\end{remark}

These connections, formalized in Corollary~\ref{cor:closed_cycle_epiplexity_limit}, are most directly applicable in repeated operation with reusable memory (or other closed-cycle conventions). Outside such regimes, logically reversible computation together with unmetered low-entropy resources can, in principle, create correlations with arbitrarily small in-boundary dissipation (cf.\ Bennett~\cite{bennett1982thermodynamics}). Proposition~\ref{prop:open_boundary_decoupling} gives a simple decoupling example, included to motivate boundary closure rather than to weaken closed-cycle benchmark statements.

\begin{proposition}[Non-equivalence in open boundaries]\label{prop:open_boundary_decoupling}
Fix a temperature $T$. Let $Z\in\{0,1\}^n$ be an $n$-bit environment-instance random variable with distribution $p(z)$.
Suppose an $n$-bit memory register $M\in\{0,1\}^n$ enters the accounting boundary already initialized to a fixed value
$M^{\mathrm{pre}}=0^n$ independent of $Z$, and suppose the preparation/maintenance of this low-entropy resource is not charged
to the in-boundary energy budget.
Assume logically reversible gates (in particular CNOT/XOR) can be implemented quasistatically so that the total in-boundary
dissipated heat of the protocol can be made arbitrarily small.

Then for every $\epsilon>0$ there exists an in-boundary protocol producing a post-state $M^{\mathrm{post}}$ such that
\[
\Delta I(M;Z)\;\triangleq\; I(M^{\mathrm{post}};Z)-I(M^{\mathrm{pre}};Z)\;=\;H(Z)\;\le\;n,
\]
while the in-boundary dissipation satisfies $Q_{\mathrm{diss}}\le \epsilon$.
\end{proposition}

\begin{proof}
Since $M^{\mathrm{pre}}\!=\!0^n$ is deterministic and independent of $Z$, we have $I(M^{\mathrm{pre}};\allowbreak Z)=0$.

Consider the bijective map
$f\colon\{0,1\}^n\!\times\!\{0,1\}^n\to\{0,1\}^n\!\times\!\{0,1\}^n$
defined by
\[
f(z,m)\;=\;(z,\;m\oplus z),
\]
where $\oplus$ denotes bitwise XOR. This map is logically reversible (indeed, $f^{-1}=f$) and can be implemented by $n$ CNOT
gates that leave $Z$ unchanged and XOR each bit of $Z$ into the corresponding bit of $M$.
Applying $f$ to $(Z,M^{\mathrm{pre}})$ yields $M^{\mathrm{post}}=0^n\oplus Z=Z$ deterministically, hence
\[
I(M^{\mathrm{post}};Z)=H(M^{\mathrm{post}})-H(M^{\mathrm{post}}\mid Z)=H(Z)-0=H(Z).
\]
Therefore $\Delta I(M;Z)=H(Z)$.

By the quasistatic-reversible implementation assumption, for the given $\epsilon>0$ we can realize this reversible circuit so
that the total in-boundary dissipation satisfies $Q_{\mathrm{diss}}\le \epsilon$.
\end{proof}

\begin{remark}[Where the ``missing cost'' resides]
The protocol exploits the unmetered initialized register $M^{\mathrm{pre}}=0^n$; including its preparation cost closes the boundary and restores Landauer-scaled benchmarks. In addition, the quasistatic limit $Q_{\mathrm{diss}}\to 0$ requires diverging time; empirically, bits/J should be reported alongside bits/s under a wall-clock constraint.
\end{remark}

In the physically relevant closed-cycle regime with bounded reusable memory and fixed reliability/resolution conventions, acquiring new \emph{retained} information typically requires overwriting, refreshing, or compressing existing memory, reintroducing Landauer-scaled dissipation on average. More generally, information-thermodynamic formulations that explicitly include the information reservoir (e.g.\ Sagawa and Ueda~\cite{sagawa2013role}) restore the second law under boundary closure.

In summary, under closed-cycle boundary-closure conventions, Corollary~\ref{cor:closed_cycle_epiplexity_limit} yields a Landauer-scale thermodynamic benchmark for dissipation-normalized learning efficiency; translating it into a bound on $\eta_{\mathcal{E}}=\Delta\mathcal{I}/E_{\mathrm{cons}}$ requires the energy accounting in Remark~1. Without charging for low-entropy resources crossing the boundary, Proposition~\ref{prop:open_boundary_decoupling} shows that bits-per-joule ratios can be arbitrarily large in principle; practical systems are far below the benchmark due to algorithmic and hardware overhead.
This also aligns with results in the thermodynamics of prediction: Still et al.~\cite{still2012thermodynamics} relates dissipation to storing non-predictive information; maximizing epiplexity per joule favors capturing predictive structure and discarding noise.

\section{Empowerment per Joule: Control Information Efficiency}

\subsection{Defining Empowerment (Control Capacity)}
Beyond learning about the world, intelligent agents also act on the world. We need a quantitative measure for an agent's ability to influence its environment. Empowerment is an information-theoretic measure of an agent's potential control over future states. Formally, empowerment $\mathcal{E}_{\mathrm{emp}}(s_0)$ at a given initial state $S_0=s_0$ is defined as the maximum mutual information between a sequence of agent actions $A_{0:\tau-1}$ and the resulting state (or observations) $S_{\tau}$ after some time horizon $\tau$:
\begin{align}
  \mathcal{E}_{\mathrm{emp}}(s_0)
  \;=\;
  \max_{p(a_{0:\tau-1})}
  I\!\left(A_{0:\tau-1}; O_{\tau}\,\middle|\, S_0=s_0\right).
\end{align}
In simpler terms, empowerment is the channel capacity from the agent's action space to the future state of its sensors. It reflects how many distinct states the agent can reliably steer the world into, up to $\tau$ steps ahead. If the agent has no control (actions do nothing), empowerment is zero bits. If the agent has perfect control to set $\log_2 N$ equally likely distinguishable states, empowerment is $\log_2 N$ bits. Empowerment is intrinsically a measure of possibility, not actual reward---it measures available options and influence. Klyubin et al. \cite{klyubin2005empowerment,arxiv2502_15820} introduced empowerment as a generic, task-agnostic utility: organisms tend to seek states where they have more influence (options), as a proxy for being in a favorable situation.

\paragraph{Endpoint variable.}
Unless stated otherwise, we take the endpoint variable to be the agent's sensor observation $O_\tau$ (``observation empowerment''). When the full environment state $S_\tau$ is available (e.g.\ in a simulator), one may additionally report ``state empowerment'' by replacing $O_\tau$ with $S_\tau$; this is a stronger benchmark setting and should be flagged explicitly.

\paragraph{Empowerment as a system-level capability.}
In this paper, we intentionally interpret empowerment as a property of the embodied agent--environment interface: given the sensorimotor dynamics, horizon $\tau$, resolution convention, and energy cost $c(\cdot)$, empowerment is the channel capacity available to the system. It is not meant to score a particular controller implementation; any realized policy may attain less than this capacity, but the capacity itself captures the physical capability of the embodiment.

\paragraph{Finite-resolution and noise assumptions.}
For continuous actions/observations, $I(A_{0:\tau-1};O_\tau)$ can depend on measurement resolution and may diverge in idealized noiseless deterministic limits. In embodied physical settings, however, thermal and sensor noise impose a finite effective resolution. When needed, we can make this explicit by defining an $\varepsilon$-resolution empowerment (or ``$\varepsilon$-empowerment'') via a discretization of $O_\tau$ at resolution $\varepsilon$ (or by assuming an explicit observation noise model), ensuring the resulting mutual information is operationally finite.

\paragraph{Consistency with epiplexity.}
For meaningful cross-system comparison, we require that the empowerment coarse-graining (or noise model) be reported and held fixed across systems, analogously to the $\varepsilon$-coarse-grained epiplexity definitions in Section~2.

\subsection{Empowerment per Joule: Efficiency of Control}
To define empowerment per joule $\eta_{\mathcal{C}}$, we must couple control capacity to an explicit physical cost. Let $c(a_{0:\tau-1})$ denote the energetic cost (in joules) of executing an action sequence, including actuation and any other in-boundary energetic costs included by the accounting convention (e.g.\ sensing/compute). A principled ``bits-per-joule'' control metric is then a \emph{capacity-per-unit-cost} objective:
\begin{align}
  \eta_{\mathcal{C}}^\star
  \;\triangleq\;
  \sup_{p(a_{0:\tau-1})}
  \frac{I(A_{0:\tau-1};O_\tau)}{\mathbb{E}[c(A_{0:\tau-1})]}
  \quad [\text{bits/J}],
  \label{eq:capacity_per_cost}
\end{align}
which is the classical ``capacity per unit cost'' objective in information theory \cite{verdu1990capacity}.
In practice, the cost convention must make the optimization well-posed and must prevent ``free control'' artifacts (e.g.\ apparent information transfer arising from autonomous environment dynamics under a zero-cost ``wait''). Two common reporting conventions are:
(i) \emph{total on-boundary} energy over the horizon (including any baseline/idle draw), yielding a wall-plug style system-level bits/J; and
(ii) \emph{incremental} energy above a stated reference ``null'' action/policy, reported together with the baseline term and with an explicit restriction to action distributions with strictly positive expected cost.
The choice determines whether $\eta_{\mathcal{C}}$ is interpreted as whole-system energy efficiency or incremental interface efficiency; in all cases the accounting boundary, baseline decomposition, horizon $\tau$, and resolution/noise convention should be reported. We recommend reporting the baseline (existence/idle) energy and the incremental (control-induced) energy separately whenever possible; interpreting $\eta_{\mathcal{C}}$ without this decomposition is generally ambiguous. When reporting incremental efficiencies, restrict to action distributions with strictly positive expected incremental cost.

For embodied systems, our recommended default is total in-boundary episode energy. Count controller compute, sensing, communication, and actuation when required to realize the action-to-observation channel. Actuation-only energy can be reported as a component diagnostic, but should not be labeled as whole-agent $\eta_{\mathcal{C}}$. For shared infrastructure or simulators, report total wall-plug/PUE-adjusted energy and the rule assigning shared idle draw to the episode.

\paragraph{Estimating empowerment in practice.}
For small discrete systems, Eq.~\eqref{eq:capacity_per_cost} can be computed exactly (e.g.\ via Blahut--Arimoto). For larger or continuous systems, report a lower bound via variational mutual-information estimators or via a learned dynamics model plus discretization, and report estimator settings and uncertainty where possible.

Equivalently, one can define a cost-constrained empowerment curve
\begin{align}
  \mathcal{E}_{\mathrm{emp}}(E_0)
  \;\triangleq\;
  \max_{p(a_{0:\tau-1}):\,\mathbb{E}[c]\le E_0}
  I(A_{0:\tau-1};O_\tau),
\end{align}
and report $\eta_{\mathcal{C}}(E_0)\triangleq \mathcal{E}_{\mathrm{emp}}(E_0)/E_0$ (or the marginal slope $d\mathcal{E}_{\mathrm{emp}}/dE_0$) at a chosen operating budget. This removes an ambiguity in defining ``the energy needed to achieve empowerment'': the optimization over $p(a_{0:\tau-1})$ trades off distinguishability and energetic expense in a single well-posed problem.
We recommend treating $\mathcal{E}_{\mathrm{emp}}(E_0)$ (and its marginal slope) as the primary reporting object, since ratio-based summaries can be sensitive to the treatment of zero/near-zero cost baselines.

This metric connects to physical limits of communication and actuation: to imprint reliably distinguishable outcomes on a noisy physical environment typically requires nonzero work. In repeated (closed-cycle) operation at temperature $T$ with a fixed resolution/reliability convention, $\kB T\ln 2$ again sets a natural benchmark scale (up to constant factors), so $\eta_{\mathcal{C}}$ is plausibly on the order of $1/(\kB T\ln 2)$ bits/J as a yardstick. In practice, actuators and policies are far from this limit; wasted actuation energy or redundant action sequences reduce $I(A_{0:\tau-1};S_\tau)$ per joule.

$\eta_{\mathcal{C}}$ measures how efficiently energetic cost can be converted into reliable control-channel information under the stated resolution and cost conventions, and serves as a diagnostic of embodiment-level control efficiency.

\section{Thermodynamic Limits and Trade-offs in Energy--Information Efficiency}

Having defined our two bits-per-joule metrics, we now clarify the assumptions under which they support thermodynamic \emph{benchmark} statements and discuss trade-offs that arise in repeated operation. Unless stated otherwise, benchmark statements in this section assume the closed-cycle and boundary-closure conventions in Section~\ref{sec:accounting}; Proposition~\ref{prop:open_boundary_decoupling} highlights why these conventions matter.

\subsection{Bounded Memory and Periodic Reset}

In practice, an AI agent has finite reusable memory and operates in repeated episodes (closed-cycle). Each lasting bit must be encoded in reused degrees of freedom and extraneous information discarded; this entails an unavoidable Landauer-scaled cost per reusable bit (Corollary~\ref{cor:closed_cycle_epiplexity_limit}). Similarly, empowerment per joule is bounded because the actuation interface must be restored to realize the same control capacity.

Because learning and control draw on shared energy budgets in closed-loop operation, systems may exhibit trade-offs between high $\eta_{\mathcal{E}}$ and high $\eta_{\mathcal{C}}$; reporting both metrics makes such trade-offs visible.

\subsection{Stochastic-Thermodynamic Information Budget}

Beyond the learning bounds in Section~2, thermodynamic limits on control arise when actions reliably constrain future environmental states: under fixed coarse-graining and reliability, imposing $\Delta I$ bits of constraint requires nonzero work, with a natural lower-bound scale on the order of $\kB T\ln 2\cdot \Delta I$ up to constant factors and efficiency losses.

We can combine these insights into a useful \emph{benchmark} statement for repeated operation. Consider an isothermal closed-loop agent--environment interaction at temperature $T$ with an explicit accounting boundary that includes the agent's reusable memory and actuation interface. Let $\Sigma_{\text{tot}} \ge 0$ denote the total entropy production of the agent+environment+heat bath over an episode. In many standard formulations of information thermodynamics, the creation of correlations (mutual information) between memory and environment and the reduction of uncertainty relevant for prediction/control are constrained by $\Sigma_{\text{tot}}$ (possibly up to additional free-energy and internal-entropy terms, as in Eq.~\eqref{eq:gs_bound_bits}).

Under a simplified closed-cycle benchmarking regime in which (i) the agent and environment return to the same marginal state distributions across episodes (so net internal Shannon entropy changes average to zero) and (ii) no additional free-energy resources are injected other than the accounted work/heat flows, one expects a coarse information budget of the form
\begin{align}
  \Delta I_{\mathrm{agent}} + \Delta I_{\mathrm{env}}
  \;\lesssim\;
  \frac{\Sigma_{\mathrm{tot}}}{\kB\ln 2}
  \;\approx\;
  \frac{Q_{\mathrm{diss}}}{\kB T\ln 2},
\end{align}
where $\Delta I_{\mathrm{agent}}$ corresponds to epiplexity gain $\Delta I(W;Z)$ and $\Delta I_{\mathrm{env}}$ corresponds to control-channel information (empowerment usage) over the same episode. This is not claimed as a universal identity without assumptions; rather, it serves as a \emph{conceptual yardstick} for how a fixed dissipation budget must be divided between learning and control in closed-cycle operation.

The above discussion formalizes the intuition that recognition (information intake) and control (information output) are thermodynamically dual processes, each constrained by entropy costs. Achieving a balance is key to building physically intelligent systems that don't ``burn out'' energetically.

\section{Unified Efficiency Framework for Physical AI Agents}\label{sec:unified_framework}

Having defined thermodynamic epiplexity and empowerment as separate metrics, we now propose a unified \emph{efficiency} framework for evaluating AI systems in embodied, closed-loop scenarios. Consider an agent (robot or AI system) interacting with an environment continuously. At any given time, the agent receives observations and rewards (if any) and takes actions. We can track two cumulative quantities over an interval: (1) the structural information the agent has acquired in its internal state (e.g. how much its model or belief has been refined -- epiplexity gained), and (2) the control information it has embedded in the environment (e.g. how much it has changed the state of the world in a goal-agnostic information sense -- empowerment utilized). Meanwhile, the agent expends a certain amount of energy in sensing, computing, and acting.

\subsection{Minimum reporting checklist (for consistent bits/J reporting)}
Because both bits and joules depend on conventions, we recommend that empirical reports of $\eta_{\mathcal{E}}$ and $\eta_{\mathcal{C}}$ include:
\begin{itemize}
  \item \textbf{Accounting boundary:} what energy flows are included (compute, sensing, actuation, communication, cooling/PSU losses), and whether energy is wall-plug, battery draw, or component-level.
  \item \textbf{Energy balance terms:} whether $\Delta E_{\mathrm{store}}$ or exported work $W_{\mathrm{out}}$ (Eq.~\eqref{eq:energy_balance}) are negligible; if not, report them or state whether the metric uses total vs.\ incremental energy.
  \item \textbf{Baseline/null policy (for incremental control costs):} the definition of the reference ``do-nothing'' action/policy, the baseline energy over the horizon, and whether $\eta_{\mathcal{C}}$ is reported as total-energy or incremental-energy efficiency.
  \item \textbf{Coarse-graining / noise model:} the $\varepsilon$ (or measurement model) used to make mutual informations finite and comparable for continuous variables.
  \item \textbf{Horizon and sampling:} episode definition, time horizon~$\tau$, sampling period, and any reset protocol for closed-cycle evaluation.
  \item \textbf{Time/throughput:} wall-clock duration, average power, and any speed constraint; report bits/J alongside bits/s when relevant.
  \item \textbf{Estimator details:} how $\Delta I(W;Z)$ is computed when the benchmark provides~$Z$, and how MDL companions are computed otherwise~\cite{epiplexity2026}; how empowerment is approximated (capacity-per-unit-cost vs.\ cost-constrained, variational bounds, etc.).
\end{itemize}

Reporting $\eta_{\mathcal{E}}$ and $\eta_{\mathcal{C}}$ side-by-side yields a diagnostic profile of recognition vs.\ control efficiency. Because actions affect learning and learning affects control, the two metrics should be interpreted jointly under the conventions above.

\paragraph{Allocation rule and ML metrics.}
When costs are ambiguous, report (i) total measured episode energy at the widest practical boundary (wall-plug or battery draw, with PUE/cooling if applicable), (ii) a component decomposition (training amortization, inference/controller compute, sensing, communication, actuation, idle draw), and (iii) optional lifecycle terms such as manufacturing and recycling, amortized over a stated service life. Common ML metrics are special cases: bits/FLOP fixes arithmetic work as the denominator, MFU diagnoses hardware utilization inside that denominator, watts/token measures inference cost without directly measuring structure, and PUE-adjusted joules widen the boundary from device to facility.

\paragraph{Minimal MDL-style worked example.}
As a concrete instantiation, take the Pythia 410M and 1B checkpoints, trained on the same public data order for about $3.0\times10^{11}$ tokens \cite{biderman2023pythia}. Their released zero-shot LAMBADA perplexities at the final checkpoint are 10.83 and 7.92, giving code lengths $\ell_{410}=3.44$ and $\ell_{1B}=2.98$ bits/token. Over a fixed evaluation stream of $N=10^9$ tokens, the marginal compression gain is $\Delta G^{\mathrm{MDL}}=N(\ell_{410}-\ell_{1B})\approx4.5\times10^8$ bits.
With a GPU-compute boundary, $C\approx6PD_{\mathrm{train}}$ FLOPs and a stated effective efficiency $\gamma=2.5\times10^{11}$ FLOP/J, the estimated training energies are $E_{410}\approx2.9\times10^9$ J and $E_{1B}\approx7.3\times10^9$ J, so the marginal reported efficiency of this scale increase is $\Delta G^{\mathrm{MDL}}/(E_{1B}-E_{410})\approx 1.0\times10^{-1}$ bits/J. This number is convention-dependent, but it exercises the checklist: benchmark, estimator, energy boundary, baseline, horizon, throughput assumption, and exclusions are all explicit.

\section{Discussion}

\paragraph{Scaling Laws and Epiplexity Trend.}

Large-scale AI models exhibit scaling laws where predictive loss improves predictably with compute and data. To connect scaling to bits-per-joule when $Z$ is unavailable, use an operational compression-gain proxy: let $\ell(C)$ denote test cross-entropy in \emph{bits per token} as a function of training compute $C$, and let $\ell_0$ be a baseline code length. For a fixed evaluation set of $N$ tokens,
\begin{align}
  G^{\mathrm{MDL}}(C)
  \;=\;
  N\bigl(\ell_0 - \ell(C)\bigr)
  \qquad [\text{bits}].
\end{align}
If training energy satisfies $E_{\mathrm{train}} \propto C$ for a given hardware stack, then an operational bits-per-joule learning efficiency is
\begin{align}
  \eta_{\mathcal{E}}^{\mathrm{MDL}}(C)
  \;\triangleq\;
  \frac{G^{\mathrm{MDL}}(C)}{E_{\mathrm{train}}(C)}
  \quad [\text{bits/J}],
\end{align}
and the \emph{marginal} efficiency of additional compute is
\begin{align}
  \frac{dG^{\mathrm{MDL}}}{dE_{\mathrm{train}}}
  \;=\;
  -\,N\,\frac{d\ell}{dC}\,\frac{dC}{dE_{\mathrm{train}}}.
  \label{eq:marginal_bits_per_joule}
\end{align}
Under common power-law scaling $\ell(C)=\ell_\infty + aC^{-\alpha}$, the marginal compression gain per unit training energy decays as $C^{-(\alpha+1)}$ (up to hardware proportionality), making diminishing returns explicit. This is an operational statement combining an empirical scaling law with an energy model; it does not imply proximity to thermodynamic limits. Contemporary training stacks are many orders of magnitude above the $\kB T\ln 2$ scale, and observed diminishing returns are dominated by algorithmic convergence and engineering losses (hardware utilization, memory movement, cooling, etc.) rather than by thermodynamic limits. Nevertheless, reporting energy-adjusted marginal gains can make these dominant factors explicit and comparable across algorithms and systems.

\section{Conclusion}

We proposed a two-axis bits-per-joule characterization of physically grounded intelligence---thermodynamic epiplexity per joule (learning efficiency) and empowerment per joule (control efficiency)---and derived a closed-cycle Landauer-scale benchmark (Corollary~\ref{cor:closed_cycle_epiplexity_limit}) alongside a decoupling construction (Proposition~\ref{prop:open_boundary_decoupling}) clarifying why boundary closure is essential. The key message is that meaningful comparisons require the stated conventions; without them, information gain and dissipation can be decoupled.

An immediate next step is to instantiate these conventions in concrete benchmarks with fully specified boundaries, estimators, and time/throughput constraints.

\begin{credits}
\subsubsection{\ackname}
We thank Ryuichi Maruyama for useful comments. This work was supported by Advanced General Intelligence for Science Program (AGIS), the RIKEN TRIP initiative.

\subsubsection{\discintname}
The authors have no competing interests to declare that are relevant to the content of this article.
\end{credits}

\bibliographystyle{splncs04}
\bibliography{references}

\end{document}